\documentclass{article}
\usepackage{ijcai26}

\usepackage{times}
\usepackage{soul}
\usepackage{url}
\usepackage[hidelinks]{hyperref}
\usepackage[utf8]{inputenc}
\usepackage[small]{caption}
\usepackage{graphicx}
\usepackage{amsmath}
\usepackage{amssymb}
\usepackage{amsthm}
\usepackage{booktabs}
\usepackage{algorithm}
\usepackage{algorithmic}
\usepackage[switch]{lineno}
\usepackage{multirow}
\usepackage{color}
\usepackage{fancyhdr}
\newtheorem{theorem}{Theorem}
\newtheorem{lemma}{Lemma}
\newtheorem{proposition}{Proposition}
\newtheorem{assumption}{Assumption}

\title{ITBoost: Information-Theoretic Trust for Robust Boosting}

\author{
Ye Su$^{1,2}$
\and
Longlong Zhao$^{1}$\thanks{Corresponding authors}
\and
Diego García-Gil$^{3}$\footnotemark[1]
\and
Jipeng Guo$^4$\and
Gangchun Zhang$^1$\and
Jinxin Chen$^1$\and
Jinsong Chen$^1$\\
\affiliations
$^1$Shenzhen Institutes of Advanced Technology, Chinese Academy of Sciences\\
$^2$University of Chinese Academy of Sciences\\
$^3$Department of Software Engineering, Andalusian Research Institute in Data Science and Computational Intelligence (DaSCI), University of Granada\\
$^4$College of Information Science and Technology, Beijing University of Chemical Technology\\
\emails
\{ye.su, ll.zhao, gc.zhang, jx.chen2, js.chen\}@siat.ac.cn,
djgarcia@ugr.es,
guojipeng@buct.edu.cn
}

\begin{document}

\maketitle

\begin{abstract}
    Gradient boosting remains a strong and widely used method for tabular data learning, but its performance often degrades when training labels are noisy. This behavior is largely related to the way boosting algorithms emphasize samples with large gradients, without explicitly accounting for whether such errors originate from informative hard cases or from unreliable labels. We address this issue by reconsidering how sample reliability is evaluated during boosting. Instead of relying on instantaneous error, we examine the evolution of each sample’s residuals across iterations. Based on this insight, we propose Information-Theoretic Trust Boosting (ITBoost), which uses the Minimum Description Length principle to measure the complexity of residual trajectories. Samples whose residual patterns fluctuate in an irregular manner are treated as less trustworthy and are down-weighted during learning. Theoretically, we derive a tighter generalization bound for ITBoost under label noise. Empirical results on various tabular benchmarks indicate that ITBoost provides improved robustness in noisy environments over leading boosting and deep tabular models, while retaining best average performance on clean data.
\end{abstract}

\section{Introduction}

Gradient Boosting Decision Tree (GBDT) have emerged as the standard for structured data modeling, delivering state-of-the-art performance across a wide range of machine learning tasks \textcolor{blue}{\cite{bentejac2021comparative,shwartz2022tabular}}. The effectiveness of these algorithms lies in their iterative learning paradigm, specifically, by sequentially training weak learners to fit the residuals of the ensemble formed by previous learners \textcolor{blue}{\cite{friedman2001greedy}}. A core driver of this process is the model’s continuous effort to address its own prediction errors: the magnitude of the gradient essentially guides where the model focuses its learning attention in the next iteration.

Yet, this focus on prioritizing large gradients stands out as a key flaw in gradient boosting. Label noise often causes sharp, even catastrophic, performance declines \textcolor{blue}{\cite{long2008random,miao2015rboost,einziger2019verifying}}. Mislabeled samples are persistent sources of error, and their large gradients slowly lead the model astray. \textcolor{blue}{Figure \ref{fig:1}} clearly shows this problem: as noise levels rise from 10\% (a) to 30\% (b), the decision boundary of a standard GBDT becomes fragmented. Here, the model’s rigid reliance on gradient magnitude as a guiding cue leads it to distort its structure to fit noisy samples (marked "X"), a choice that severely undermines generalization. This is not just an edge case: label noise is far from uncommon in practical fields like medicine and finance. As such, GBDT’s vulnerability poses a major obstacle to its real-world use \textcolor{blue}{\cite{frenay2013classification,shi2024survey}}.

\begin{figure*}[ht]
    \centering
    \includegraphics[width=0.85\textwidth]{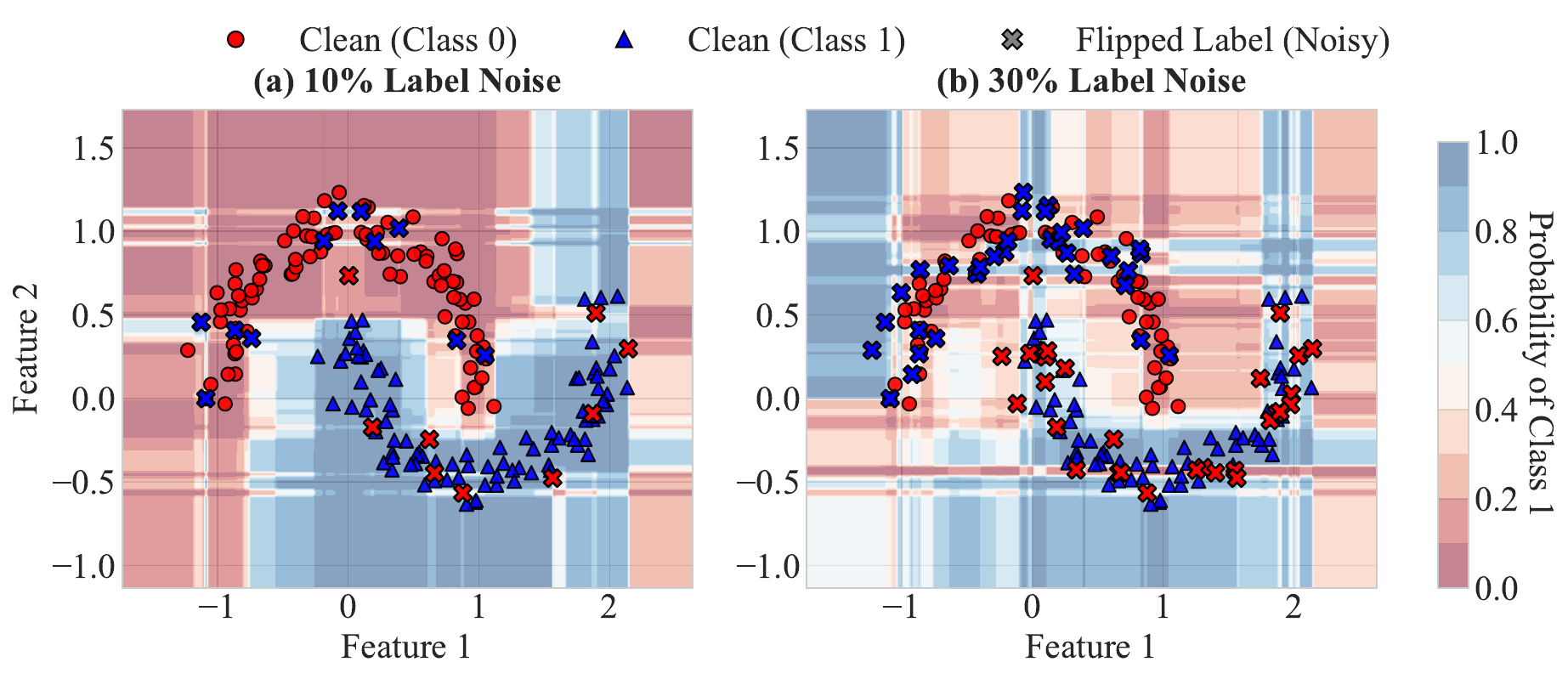}
    \caption{Standard GBDT overfitting noise on data with 10\% and 30\% label noise.}
    \label{fig:1}
\end{figure*}

Current efforts to tackle this issue rely on either robust loss functions \textcolor{blue}{\cite{huber1992robust,zhang2018generalized}} or sample selection strategies \textcolor{blue}{\cite{ferreira2012boosting,dubout2014adaptive,liu2020robust}}, but these approaches fall short of addressing the core issue. They still rest on the flawed premise that gradient magnitude alone dictates a sample’s relevance to the true decision boundary. This presents a challenge that’s hard to resolve: algorithms cannot tell apart two types of samples that both yield large gradients—genuinely valuable hard examples sitting close to the decision boundary, and erroneous noisy samples with incorrect labels \textcolor{blue}{\cite{song2022learning}}. If we blunt all large gradients outright, we end up hampering the model’s ability to learn from critical hard examples; if we accept them uncritically, noise will inevitably corrupt the model’s learning process.

In response to this dilemma, we question this foundational assumption and advance a new perspective: a sample’s gradient reliability should not hinge on its immediate magnitude, but rather on how interpretable its error history is. Hard examples, while tough to learn, produce residual sequences with recognizable structural patterns. By contrast, noisy examples, intrinsically misaligned with the true data distribution, result in far more erratic prediction errors.
Drawing on this observation, we introduce Information-Theoretic Trust Boosting (ITBoost). ITBoost assigns a “trust score” to each sample’s gradient. Instead of only asking, “How large is the sample’s error?”, we shift focus to: “How trustworthy is the historical pattern of its errors?” To formalize this question, we adopt the Minimum Description Length (MDL) principle from information theory \textcolor{blue}{\cite{rissanen1978modeling,grunwald2007minimum}}. Concretely, we compute the sequential complexity of each sample’s residual trajectory online, using the Lempel-Ziv algorithm \textcolor{blue}{\cite{gasieniec1996efficient}} to quantify randomness. A random, uncompressible residual sequence (high complexity) flags the sample as “untrustworthy,” prompting us to reduce its weight in the learning process. This new weighting mechanism lets ITBoost effectively de-emphasize noise while still prioritizing genuinely challenging yet informative samples. Our key contributions are summarized below:

\begin{itemize}
    \item We introduce ITBoost, a robust boosting method that tackles label noise at its root by leveraging an information-theoretic trust framework, an innovative departure from standard boosting.
    \item We develop an MDL-based sample weighting mechanism; uniquely, this mechanism uses residual time-series characteristics instead of instantaneous values to guide the boosting process.
    \item We provide theoretical justification for ITBoost, demonstrating that it achieves a tighter generalization error bound than standard GBDT when label noise is present.
    \item We show that ITBoost delivers optimal performance and robustness, outperforming leading boosting algorithms and modern deep tabular models in noisy environments.
\end{itemize}

\section{Information-Theoretic Boosting}

Our work is based on the insight that instantaneous gradient magnitude is an insufficient metric, and often misleading, for assessing a sample’s true importance. Indeed, a large gradient may indicate just as readily a valuable, hard-to-learn case as it does a corrupted, noisy counterpart. To clear up this ambiguity, we propose leveraging the full history of a sample’s learning trajectory to underpin our assessment of its importance.

\subsection{From Gradient Magnitude to Residual History}
\label{ssec:res_hist}

Let us frame the learning process through a temporal lens. For each sample $x_i$, the sequence of pseudo-residuals, $\mathbf{g}_i = (g_i^{(1)}, g_i^{(2)}, \dots, g_i^{(M)})$, constitutes a time series documenting how the model’s understanding of that sample evolves over iterations. What we argue here is that the traits of this time series hold far more information about the sample’s inherent nature than the magnitude of any single residual element alone.

\begin{itemize}
    \item \textbf{Hard but Clean Samples.} Authentic hard samples exhibit low-complexity residual structures. Specifically, \textit{(i)} boundary samples induce structured sign oscillations as the decision boundary is progressively refined, while \textit{(ii)} persistent hard samples generate stable, monotonic residual sequences due to systematic model underestimation. Although these samples are difficult to learn, both residual patterns are highly compressible and thus have low algorithmic complexity. Guided by the MDL principle, ITBoost correctly identifies them as reliable and assigns higher weights, directing the model toward these valid yet challenging cases.
    
    \item \textbf{A Noisy Sample.} Samples with corrupted labels conflict with the underlying feature structure. Attempts to fit such data lead to erratic and unpredictable model adjustments, producing residual sequences with chaotic sign flips and no discernible pattern, indicative of high complexity. This behavior clearly distinguishes noisy samples from both the structured oscillations of boundary samples and the stable persistence of hard clean samples.
\end{itemize}

It is this fundamental distinction in the temporal signature of residual sequences that forms the core insight underpinning our work. An unpredictable, random-like sequence is a strong indicator that the model’s errors stem from label corruption, rather than being an integral part of the structured learning process.

\subsection{Quantifying Trust with Algorithmic Complexity}

To formalize this intuition, we measure the randomness of a sequence using its algorithmic complexity. The Kolmogorov complexity \textcolor{blue}{\cite{kolmogorov1965three,li2008introduction}}, $K(s)$, of a sequence $s$ is the length of the shortest program that generates $s$. Since $K(s)$ is uncomputable, we employ a practical approximation: the Lempel-Ziv (LZ) complexity, which effectively measures a sequence's compressibility \textcolor{blue}{\cite{wyner2002sliding,zozor2005lempel}}.

At each iteration $m$, we consider the binarized residual history for each sample $i$, $\mathbf{s}_i^{(m)} = (\text{sign}(g_i^{(1)}), \dots, \text{sign}(g_i^{(m)}))$. This binarization is a crucial design choice, as it intentionally discards the noisy and volatile error magnitude to focus purely on the temporal pattern of the error's direction---a more stable and robust signal of a sample's learning trajectory. We then compute its LZ complexity, $C(\mathbf{s}_i^{(m)})$. We define the information-theoretic trust weight, $w_i^{(m)}$, as the product of the gradient magnitude and a trust term, $\tau_i^{(m)}$:
\begin{equation}
    w_i^{(m)} = |g_i^{(m)}| \cdot \tau_i^{(m)},
    \label{eq:5}
\end{equation}
where the trust term is an exponential penalty based on the normalized complexity of its history:
\begin{equation}
    \tau_i^{(m)} = \exp \left( - \bar{C}(\mathbf{s}_i^{(m)}) \right).
    \label{eq:6}
\end{equation}
Here, $\bar{C}(\mathbf{s}_i^{(m)})$ is the LZ complexity, normalized to the range $[0,1]$ across all samples at iteration $m$. This formulation allows ITBoost to temper the influence of samples that have proven to be historically unreliable.

To compute this normalized complexity, we employ a dynamic Min-Max scaling procedure at each boosting iteration $m$. For the set of raw LZ complexities $\{C_1, C_2, \dots, C_N\}$ computed at step $m$, the normalized complexity for sample $i$ is given by:
\begin{equation}
    \bar{C}_i = \frac{C_i - \min(\{C_k\}_{k=1}^N)}{\max(\{C_k\}_{k=1}^N) - \min(\{C_k\}_{k=1}^N)}.
    \label{eq:7}
\end{equation}
This per-iteration, relative normalization is a critical design choice. A key implication is its adaptive response to outliers. It pins the noisy sample's normalized complexity near 1.0, ensuring it receives a significant trust penalty via the exponential term $e^{-\bar{C}_i}$. It compresses the normalized complexities of most clean samples into a narrow band close to 0, thereby shielding them from undue penalization. This dynamic, outlier-driven normalization thus enhances the separation between trustworthy and untrustworthy samples in the trust domain, enabling a more targeted and aggressive suppression of noise.

\subsection{Algorithm Summary}

The complete ITBoost algorithm is presented in \textcolor{blue}{Algorithm \ref{alg:1}}. The key modification to the standard GBDT framework is the introduction of Step 3, where the information-theoretic trust weights are computed. These weights are then used to guide the training of the weak learner in Step 4, replacing the implicit weighting by gradient magnitude in the standard squared-error loss minimization.

\begin{algorithm}[tb]
    \caption{The ITBoost Algorithm}
    \label{alg:1}
    \textbf{Input}: Training data $\mathcal{D} = \{(x_i, y_i)\}_{i=1}^N$, loss $\mathcal{L}(y, F)$, iterations $M$, learning rate $\nu$.
    \begin{algorithmic}[1] 
        \STATE Initialize $F_0(x) \leftarrow \arg\min_c \sum_{i=1}^N \mathcal{L}(y_i, c)$.
        \STATE Initialize empty residual histories $s_i \leftarrow$ "" for $i=1, \dots, N$.
        \FOR{$m=1$ to $M$}
            \STATE \textbf{Step 1. Compute pseudo-residuals:} 
            \STATE \quad $g_i^{(m)} \leftarrow - \left[ \frac{\partial \mathcal{L}(y_i, F(x_i))}{\partial F(x_i)} \right]_{F(x)=F_{m-1}(x)}$
            \STATE \textbf{Step 2. Update binarized residual histories:} 
            \STATE \quad $s_i \leftarrow s_i \circ ('1' \text{ if } g_i^{(m)} > 0 \text{ else } '0')$
            \STATE \textbf{Step 3. Compute information-theoretic trust weights:}
            \STATE \quad Compute LZ complexity for each history: $C_i \leftarrow \text{LZ-Complexity}(s_i)$.
            \STATE \quad Normalize complexities across samples: $\bar{C}_i \leftarrow \text{Normalize}(C_1, \dots, C_N) \text{ to } [0,1]$.
            \STATE \quad Compute final weights using Eq. (\ref{eq:5}): $w_i^{(m)} \leftarrow |g_i^{(m)}| \cdot \exp(-\bar{C}_i)$.
            \STATE \textbf{Step 4. Train weak learner $h_m(x)$ on pseudo-residuals using trust weights:}
            \STATE \quad $h_m \leftarrow \arg\min_{h \in \mathcal{H}} \sum_{i=1}^N w_i^{(m)} \cdot (g_i^{(m)} - h(x_i))^2$
            \STATE Update the ensemble: $F_m(x) \leftarrow F_{m-1}(x) + \nu \cdot h_m(x)$.
        \ENDFOR
        \STATE \textbf{Output}: $F_M(x)$.
    \end{algorithmic}
\end{algorithm}

\section{Theoretical Analysis}

In this section, we provide theoretical justification for the robustness of ITBoost. We demonstrate that our information-theoretic weighting scheme leads to a tighter generalization error bound under a standard model of label noise compared to traditional GBDT.

\subsection{Assumptions}

Let $\mathcal{F}$ be the function class of the final ensemble, which is the convex hull of the base learner class $\mathcal{H}$. We analyze the generalization error, $\mathbb{E}_{x,y}[\ell(y, F(x))]$, where $\ell$ is the 0-1 loss. Our analysis will be based on a surrogate loss, the logistic loss $\mathcal{L}(y, F) = \log ( 1 + \exp ( - yF ))$, for $y \in \{ - 1,1\}$, which is convex and 1-Lipschitz. We adopt a standard model for label noise:

\begin{assumption}[Symmetric Label Noise]
The observed labels $y_i$ in the training set $\mathcal{D}$ are generated from the true (clean) labels $y_i^*$ according to an independent process, such that for a noise rate $p \in [0,1/2)$:
\[
P(y_i \neq y_i^*) = p \quad \text{and} \quad P(y_i = y_i^*) = 1 - p.
\]
\end{assumption}

Our key insight is that noisy samples will, with high probability, exhibit more complex residual histories. We formalize this by assuming that the expected complexity of a noisy sample's residual history is higher than that of a clean sample.

\begin{assumption}[Complexity Separation]
\label{ass:complexity_separation}
Let $C(\mathbf{s}_i)$ be the LZ-complexity of a sample's residual history. There exists a gap $\Delta_C > 0$ such that the expected complexity conditioned on the label being noisy versus clean is separated:
\begin{equation}
\mathbb{E}[C(\mathbf{s}_i)|y_i \neq y_i^*] \geq \mathbb{E}[C(\mathbf{s}_i)|y_i = y_i^*] + \Delta_C.
\label{eq:8}
\end{equation}
\end{assumption}

\textbf{Remark (Scope of Noise Types).} It is important to clarify that \textcolor{blue}{Assumption~\ref{ass:complexity_separation}} is specifically tailored to stochastic or unstructured label noise, the primary type of label noise observed in practical real-world scenarios. In such cases, conflicts between features and their corresponding labels introduce randomness into the model’s learning dynamics. A notable limitation, however, emerges when dealing with systemic or adversarial noise. This kind of noise tends to generate stable, low-complexity residual patterns, which bear resemblance to the persistent hard samples discussed in \textcolor{blue}{Section~\ref{ssec:res_hist}}. Notably, without clean validation data or additional external supervision, purely data-driven methods cannot inherently distinguish such structured corruptions from genuine hard examples. Consequently, this category of structured noise falls outside the scope of this work.

This assumption posits that our complexity measure is a meaningful signal for identifying noise, a hypothesis we validate empirically in \textcolor{blue}{Section~\ref{sec:exp}}. In Eq. (\ref{eq:8}), we assumed that the expected residual complexity of noisy samples is strictly greater than that of clean samples. While this assumption aligns with intuitive reasoning, it is relatively strong and lacks explicit conditions that guarantee the separation $\Delta_C > 0$. To this end, we reformulate the original assumption into two equivalent, verifiable variants, which also allow for finite-sample probabilistic guarantees.

\begin{assumption}[Entropy-Rate or Empirical Separability]
Let each sample $i$ produce a residual-sign sequence $S_i^{(T)} = (s_{i,1}, ..., s_{i,T})$ with normalized complexity $\widehat{C}_T(S_i) \in [0,1]$. We assume one of the following holds:

\textbf{(A) Entropy-rate gap.} Residual sequences for clean and noisy samples can be modeled as two stationary $\psi$-mixing processes with entropy rates $H_{\text{clean}}$ and $H_{\text{noisy}}$, satisfying $H_{\text{noisy}} - H_{\text{clean}} \geq \Delta_H > 0$. Under the standard Lempel--Ziv consistency theorem \textcolor{blue}{\cite{wyner2002sliding}}, the normalized complexity estimator $\widehat{C}_T$ converges exponentially to the true entropy rate:
\begin{equation}
Pr ( |\widehat{C}_T - H| \geq \epsilon ) \leq c_1 e^{-c_2 T \epsilon^2},
\label{eq:9}
\end{equation}
so for sufficiently large $T$, we obtain a probabilistic separation $\mathbb{E}[\widehat{C}_T \mid \text{clean}] \geq \Delta_c > 0$ with $\Delta_c \simeq \Delta_H - 2\epsilon_T$.

\textbf{(B) Empirical separability.} For finite samples, if $\widehat{C}_T \in [0,1]$, let $\bar{C}_{\text{clean}}$ and $\bar{C}_{\text{noisy}}$ denote the empirical means. Then, by Hoeffding's inequality,
\begin{equation}
Pr ( |\bar{C} - \mathbb{E}\widehat{C}| > \epsilon ) \leq 2e^{-2n\epsilon^2}.
\label{eq:10}
\end{equation}
Hence, whenever the true expectation gap $\Delta_c^{\star} > 2\epsilon$, the empirical means are separable with probability at least $1 - \delta$, provided $n \geq \frac{1}{2\epsilon^2}\log \frac{2}{\delta}$.
\end{assumption}

These two formulations replace the untestable Kolmogorov-complexity gap by measurable or statistically verifiable criteria. Empirical separability (B) can be directly checked in experiments, while the entropy-rate formulation (A) connects to the asymptotic theory of universal compression.

\subsection{Generalization Bound under Label Noise}

With these assumptions, we first establish the statistical properties of our trust weights. The following lemma characterizes the behavior of the trust term $\tau$ derived from the complexity random variable.

\begin{lemma}[Hoeffding / sub-Gaussian bounds for $\tau$]
\label{lemma:bounds}
Let $C$ be a real-valued random variable with mean $\mu = \mathbb{E}[C]$ and define $\tau := \mathbb{E}[e^{-C}]$. Then:
\begin{enumerate}
    \item \textbf{Jensen lower bound:} Because $f(x) = e^{-x}$ is convex, $\tau \geq e^{-\mu}$.
    \item \textbf{Bounded-range upper bound:} If $C \in [a, b]$ with range $R = b-a$, Hoeffding's lemma gives $\tau \leq \exp(-\mu + R^2/8)$.
    \item \textbf{Sub-Gaussian upper bound:} If $C - \mu$ is $\sigma^2$-sub-Gaussian, then $\tau \leq \exp(-\mu + \sigma^2/2)$.
\end{enumerate}
\end{lemma}

This lemma makes explicit that the expected trust weight $\tau$ is tightly constrained by the expected complexity $\mu$. Specifically, an increase in residual complexity brings about an exponential decay in the upper bound of the trust weight—which in turn furnishes the mathematical basis needed to justify the use of complexity scores for down-weighting samples.

\begin{proof}
    The detailed proof is provided in \textcolor{blue}{Appendix A.1} in supplementary material, please.
\end{proof}

Based on \textcolor{blue}{Lemma \ref{lemma:bounds}}, we quantify the relative weighting of noisy versus clean samples.

\begin{proposition}[Ratio bound for noisy vs. clean samples]
\label{prop:ratio}
Let $C_{\text{noisy}}$ and $C_{\text{clean}}$ denote the residual-complexity variables of noisy and clean samples, with expectations $\mu_{\text{noisy}}$ and $\mu_{\text{clean}}$. Define the complexity gap $\Delta_c = \mu_{\text{noisy}} - \mu_{\text{clean}} > 0$. Then, if both $C$'s lie in $[a, b]$ or are $\sigma^2$-sub-Gaussian:
\begin{equation}
    \frac{\tau_{\text{noisy}}}{\tau_{\text{clean}}} \leq \exp \left(-\Delta_c + \text{corr} \right),
\end{equation}
where $\text{corr}$ is a small correction term related to the variance or range of $C$.
\end{proposition}

What this proposition makes clear is that, as long as the complexity gap $\Delta_c$ surpasses the correction term (derived from variance), noisy samples will be down-weighted exponentially in comparison to clean ones. This mechanism lays theoretical groundwork for the algorithm’s capacity to suppress noise influence, guiding gradient updates to align more closely with the direction dictated by the clean data distribution.

\begin{proof}
    We provide the full derivation in \textcolor{blue}{Appendix A.2} in supplementary material, please.
\end{proof}

Finally, we state our main theoretical result regarding the generalization error.

\begin{theorem}[Quantitative Effect of Complexity Separation]
\label{thm:main}
Assume the loss $\ell$ is $L$-Lipschitz, weak learners are bounded by $B$, and the trust-weight mapping is Lipschitz. Under Assumption 3, if $\Delta_c > 0$ and complexities are sub-Gaussian, then the expected empirical risk improvement per boosting iteration satisfies:
\begin{equation}
    \mathbb{E}[R_{t+1}(F) - R_t(F)] \leq -L L_g B \Delta_c + O\left(\frac{1}{\sqrt{n}}\right).
\end{equation}
\end{theorem}

\begin{table*}[t]
    \centering
    \small 
    \caption{Performance comparison on the original binary classification datasets (no synthetic noise added).}
    \label{tab:1}
    \renewcommand{\arraystretch}{0.8} 
    \begin{tabular*}{\textwidth}{@{\extracolsep{\fill}}cccccccccc}
        \toprule
        Dataset & Metric & ADA & GBDT & XGB & LGBM & CAT & NGB & TabPFN & ITBoost \\
        \midrule
        \multirow{4}{*}{OVA\_Uterus}
        & ACC ($\uparrow$) & 0.9073 & 0.8954 & 0.9033 & 0.9194 & 0.9194 & 0.8548 & \underline{0.9237} & \textbf{0.9438} \\
        & F1 ($\uparrow$) & 0.9077 & 0.8956 & 0.9046 & 0.9201 & 0.9206 & 0.8554 & \underline{0.9236} & \textbf{0.9453} \\
        & AUC ($\uparrow$) & 0.9667 & 0.9416 & 0.9661 & \textbf{0.9777} & 0.9731 & 0.9050 & \underline{0.9758} & 0.9675 \\
        & Log Loss ($\downarrow$) & 0.4825 & 0.6694 & 0.2642 & 0.2476 & \underline{0.2277} & 0.5822 & 0.2432 & \textbf{0.1357} \\
        \midrule
        \multirow{4}{*}{Madelon}
        & ACC ($\uparrow$) & 0.6135 & 0.7335 & 0.8146 & 0.8196 & 0.8438 & 0.7335 & \underline{0.9050} & \textbf{0.9204} \\
        & F1 ($\uparrow$) & 0.6122 & 0.7404 & 0.8155 & 0.8203 & 0.8442 & 0.7408 & \underline{0.9050} & \textbf{0.9200} \\
        & AUC ($\uparrow$) & 0.6600 & 0.8117 & 0.8843 & 0.8935 & \underline{0.9087} & 0.8107 & \textbf{0.9666} & \underline{0.9542} \\
        & Log Loss ($\downarrow$) & 0.6700 & 0.5397 & 0.4522 & 0.4129 & 0.4064 & 0.5506 & \underline{0.2326} & \textbf{0.1643} \\
        \midrule
        \multirow{4}{*}{Jasmine} 
        & ACC ($\uparrow$) & 0.7922 & 0.8090 & 0.7962 & 0.8046 & 0.8090 & 0.7999 & \underline{0.8324} & \textbf{0.9179} \\
        & F1 ($\uparrow$) & 0.8128 & 0.8201 & 0.8114 & 0.8214 & 0.8278 & 0.8278 & \underline{0.8489} & \textbf{0.9185} \\
        & AUC ($\uparrow$) & 0.8439 & 0.8574 & 0.8577 & 0.8608 & 0.8679 & 0.8539 & \underline{0.8930} & \textbf{0.9415} \\
        & Log Loss ($\downarrow$) & 0.5481 & 0.4232 & 0.5027 & 0.4419 & 0.4134 & 0.4424 & \underline{0.3780} & \textbf{0.1814} \\
        \midrule
        \multirow{4}{*}{Bioresponse}
        & ACC ($\uparrow$) & 0.7582 & 0.7851 & 0.7982 & 0.7984 & 0.7918 & 0.7665 & \underline{0.8062} & \textbf{0.9174} \\
        & F1 ($\uparrow$) & 0.7773 & 0.8059 & 0.8150 & 0.8165 & 0.8113 & 0.7915 & \underline{0.8209} & \textbf{0.9238} \\
        & AUC ($\uparrow$) & 0.8219 & 0.8536 & 0.8675 & 0.8719 & 0.8626 & 0.8316 & \underline{0.8776} & \textbf{0.9567} \\
        & Log Loss ($\downarrow$) & 0.6164 & 0.4759 & 0.5460 & 0.4620 & 0.4611 & 0.5168 & \underline{0.4413} & \textbf{0.1844} \\
        \midrule
        \multirow{4}{*}{Creditcard}
        & ACC ($\uparrow$) & 0.9655 & 0.9634 & 0.9644 & 0.9685 & \underline{0.9705} & 0.9370 & 0.9614 & \textbf{0.9726} \\
        & F1 ($\uparrow$) & 0.9655 & 0.9631 & 0.9642 & 0.9679 & \underline{0.9701} & 0.9341 & 0.9605 & \textbf{0.9726} \\
        & AUC ($\uparrow$) & 0.9881 & 0.9901 & 0.9910 & \underline{0.9918} & \textbf{0.9921} & 0.9851 & 0.9919 & 0.9914 \\
        & Log Loss ($\downarrow$) & 0.4850 & 0.1171 & 0.1223 & 0.1260 & \textbf{0.0969} & 0.1647 & 0.1046 & \underline{0.1076} \\
        \midrule
        \multirow{4}{*}{Average}
        & ACC ($\uparrow$) & 0.8073 & 0.8373 & 0.8553 & 0.8621 & 0.8669 & 0.8183 & \underline{0.8857} & \textbf{0.9344} \\
        & F1 ($\uparrow$) & 0.8151 & 0.8450 & 0.8621 & 0.8692 & 0.8748 & 0.8299 & \underline{0.8918} & \textbf{0.9360} \\
        & AUC ($\uparrow$) & 0.8561 & 0.8909 & 0.9133 & 0.9191 & 0.9201 & 0.8773 & \underline{0.9410} & \textbf{0.9623} \\
        & Log Loss ($\downarrow$) & 0.5604 & 0.4451 & 0.3775 & 0.3381 & 0.3211 & 0.4513 & \underline{0.2799} & \textbf{0.1547} \\
        \bottomrule
    \end{tabular*}
\end{table*}

Theorem \ref{thm:main} indicates that the generalization risk improvement is explicitly dependent on the complexity gap $\Delta_c$. The term $-L L_g B \Delta_c$ represents a risk reduction derived from the separation between noisy and clean complexities. This suggests that larger separability contributes to more robust convergence, offering a theoretical basis for ITBoost's performance in noisy environments compared to standard GBDT which lacks this complexity-aware regularization.

\begin{proof}
    The complete proof is deferred to \textcolor{blue}{Appendix A.3} in supplementary material, please.
\end{proof}

\section{Experiments}
\label{sec:exp}

In this section, we conduct a comprehensive empirical evaluation to validate the performance of ITBoost. The primary goal of this section is to answer a critical question: Does our information-theoretic trust mechanism, designed for robustness, maintain optimal performance on clean and noise datasets? (\textcolor{blue}{Section \ref{sec:overall}} and \textcolor{blue}{\ref{sec:robustness}}). Additionally, we conduct an ablation study for residual binarization (\textcolor{blue}{Section \ref{sec:ablation}}) and visualize the information-theoretic trust mechanism of ITBoost (\textcolor{blue}{Section \ref{sec:case_study}}). Additionally, we provide the complexity and efficiency analysis (\textcolor{blue}{Section B}) in supplementary material. To ensure statistical significance, all comparative results are validated using Friedman tests (\textcolor{blue}{Section C}) in supplementary material. 

\textbf{Implementation and Reproducibility}. All experiments were conducted locally using Jupyter Notebook in an Anaconda environment on Windows 10 (Version 10.0.19045). The hardware infrastructure utilized an Intel processor (Intel64 Family 6 Model 151 Stepping 2, GenuineIntel) equipped with 20 logical cores. The software stack was built on Python 3.12.4, utilizing the following key libraries: pandas 2.2.2, numpy 1.26.4, scikit-learn 1.7.2, statsmodels 0.14.2, imbalanced-learn 0.14.0, lightgbm 4.3.0, xgboost 2.0.3, catboost 1.2.3, ngboost 0.5.7, and tabpfn 6.0.6.

\subsection{Setup}

\textbf{Dataset.} We used five different datasets from OpenML repository, which are all publicly available. These datasets include Madelon (2,600 samples, 500 features) \textcolor{blue}{\cite{guyon2007competitive}}, Jasmine (2,984 samples, 144 features), Bioresponse (3,751 samples, 1,776 features), Creditcard (284,807 samples, 30 features) \textcolor{blue}{\cite{dalpozzolo2015calibrating}}, and the OVA\_Uterus (1,545 samples, 10,936 features) \textcolor{blue}{\cite{stiglic2010stability}}. We employed random undersampling \textcolor{blue}{\cite{hasanin2018effects,liu2020dealing}} to address class imbalance in the Creditcard and OVA\_Uterus datasets.

\textbf{Baselines.} We compared ITBoost against a comprehensive suite of SOTA tabular learning algorithms. This includes six established boosting methods: AdaBoost (ADA) \textcolor{blue}{\cite{freund1997decision}}, GBDT, XGBoost (XGB) \textcolor{blue}{\cite{chen2015xgboost}}, LightGBM (LGBM) \textcolor{blue}{\cite{ke2017lightgbm}}, CatBoost (CAT) \textcolor{blue}{\cite{prokhorenkova2018catboost}}, and NGBoost (NGB) \textcolor{blue}{\cite{duan2020ngboost}}. Additionally, we included \textbf{TabPFN} \textcolor{blue}{\cite{hollmann2025accurate}}, a recent Transformer-based prior-data fitted network, to benchmark our method against modern deep tabular foundation models. \textit{Note that specialized robust boosting variants (e.g., RobustBoost) are excluded from this comparison due to the lack of maintained open-source implementations, consistent with standard practice in recent literature \textcolor{blue}{\cite{luo2025robust}}. A detailed discussion of these related works is provided in \textcolor{blue}{Appendix D}.} For a fair and reproducible comparison of the algorithms' intrinsic capabilities, we were configured with “random\_state=42” to ensure reproducibility, with all other parameters set to the official recommendation values provided by the respective software packages (e.g., Scikit-learn). The proposed ITBoost algorithm was configured in the same manner, following an official setup consistent with these baselines. We did not optimize the parameters of any algorithms. This approach evaluates the fundamental ``out-of-the-box'' performance of the underlying architectures, a crucial factor for practical applications.

\textbf{Evaluation Protocol.} We employed a 5-fold stratified cross-validation protocol for all experiments. The performance of each algorithm was reported as the mean and standard deviation across the five folds and the best results were highlighted in \textbf{bold}, second best are \underline{underlined}. The primary evaluation metrics were Accuracy (ACC), F1-Score (F1), Area Under the ROC Curve (AUC), and Logarithmic Loss (Log Loss).

\subsection{Overall Performance}
\label{sec:overall}

Table \textcolor{blue}{\ref{tab:1}} presents the performance comparison of ITBoost against seven SOTA algorithms. The results revealed a clear performance hierarchy. Among the traditional boosting families, ADA generally yielded the lowest performance, while CAT and LGBM demonstrated strong competitiveness. Notably, the recent Transformer-based TabPFN demonstrated strong performance, frequently securing the second-best results. However, ITBoost outperformed all baselines across nearly all datasets and metrics. The performance gap was particularly striking on the \textit{Madelon} dataset, where ITBoost achieves an accuracy of 92.04\% compared to 84.38\% for the best GBDT baseline (CAT) and 90.50\% for TabPFN. Since Madelon is a synthetic dataset known for its high dimensionality and abundance of uninformative, distracting features (feature noise), this result is highly significant. It suggests that our information-theoretic trust mechanism is not only effective against synthetic label flipping (as we will explore in \textcolor{blue}{Section~\ref{sec:robustness}}) but also inherently robust against intrinsic data corruption and spurious correlations. By identifying the chaotic residual patterns caused by these noisy features, ITBoost effectively acts as a dynamic noise filter on the ``clean'' data.

\begin{figure}[ht]
    \centering
    \includegraphics[width=\columnwidth]{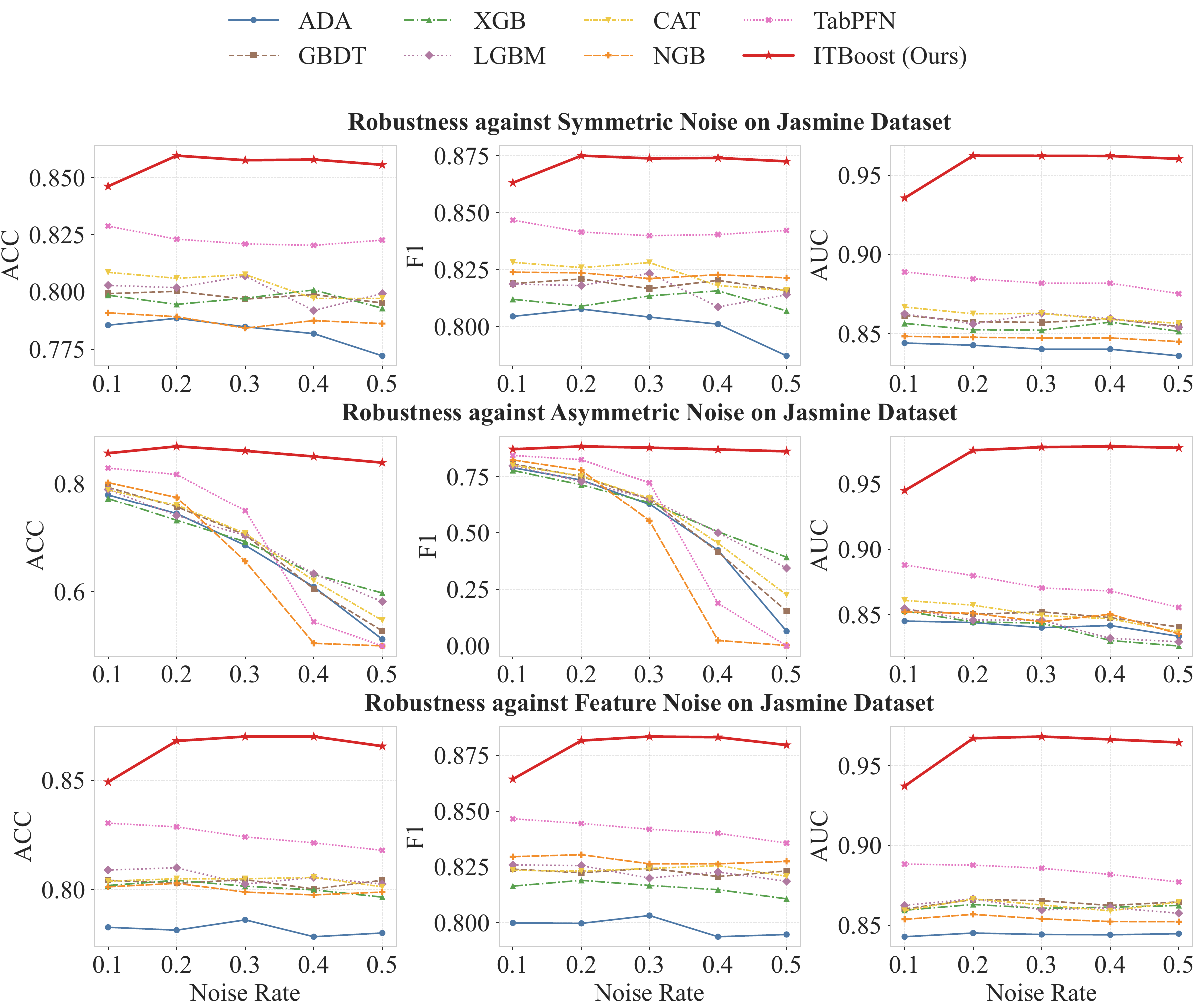}
    \caption{Robustness comparison of boosting algorithms under different types of noise on the Jasmine dataset.}
    \label{fig:2}
\end{figure}

\subsection{Robustness Analysis}
\label{sec:robustness}

\textcolor{blue}{Figure \ref{fig:2}} shows the robustness of leading boosting methods and deep tabular model (TabPFN) under three noise conditions: symmetric label noise, asymmetric label noise, and feature noise on the Jasmine dataset. These results exhibited that, as the noise rate escalates, leading boosting methods exhibit a consistent drop in performance. Notably, while TabPFN outperformed leading boosting methods, it still displayed clear vulnerability, with a more pronounced decline observed under severe asymmetric noise. However, ITBoost maintained consistently high performance even at extreme noise levels, retaining a substantial advantage over leading ensembles and TabPFN. Such empirical evidence provides robust support for our core hypothesis. Leading boosting methods are prone to being misled by spurious correlations: GBDTs rely heavily on gradient magnitude, while TabPFN’s fixed priors fail to adapt to instance-specific corruptions. ITBoost’s information-theoretic trust mechanism, by contrast, effectively distinguishes the erratic error patterns generated by noisy samples. Through systematically down-weighting the influence of these samples based on residual complexity, ITBoost preserves the model’s integrity and its capacity to capture the stable underlying data structure—ultimately achieving optimal robustness and maintaining optimal performance in scenarios where existing SOTA methods struggle or fail. 

\begin{figure}[ht]
    \centering
    \includegraphics[width=\columnwidth]{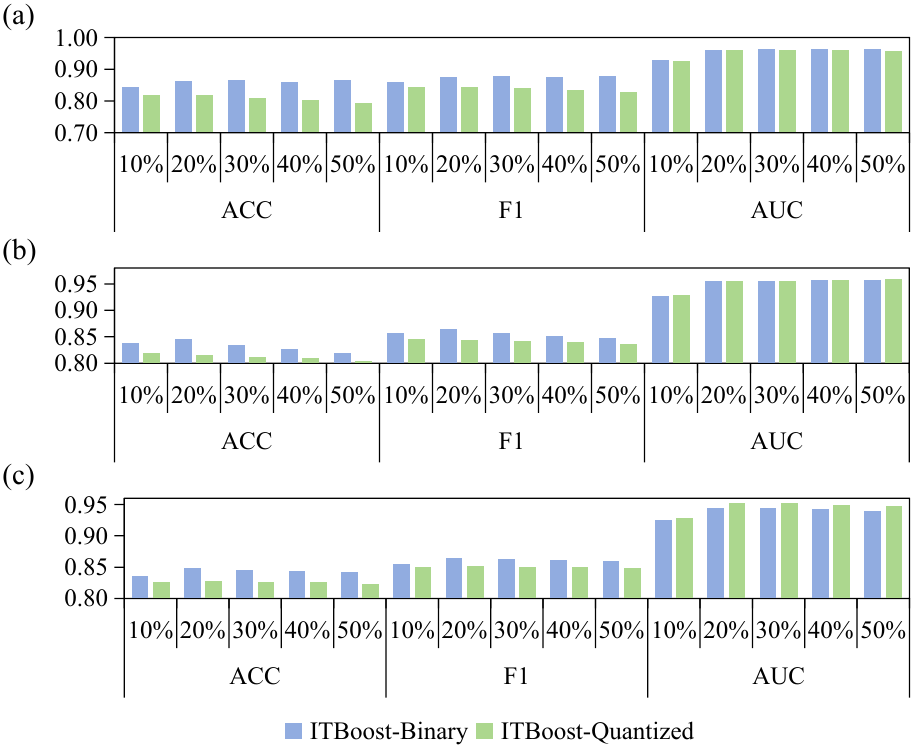}
    \caption{Ablation study for the role of residual binarization in ITBoost under different types of noise on Jasmine dataset. (a) 10\%-50\% symmetric noise. (b) 10\%-50\% asymmetric noise. (c) 10\%-50\% feature noise.}
    \label{fig:3}
\end{figure}

\begin{figure}[ht]
    \centering
    \includegraphics[width=\columnwidth]{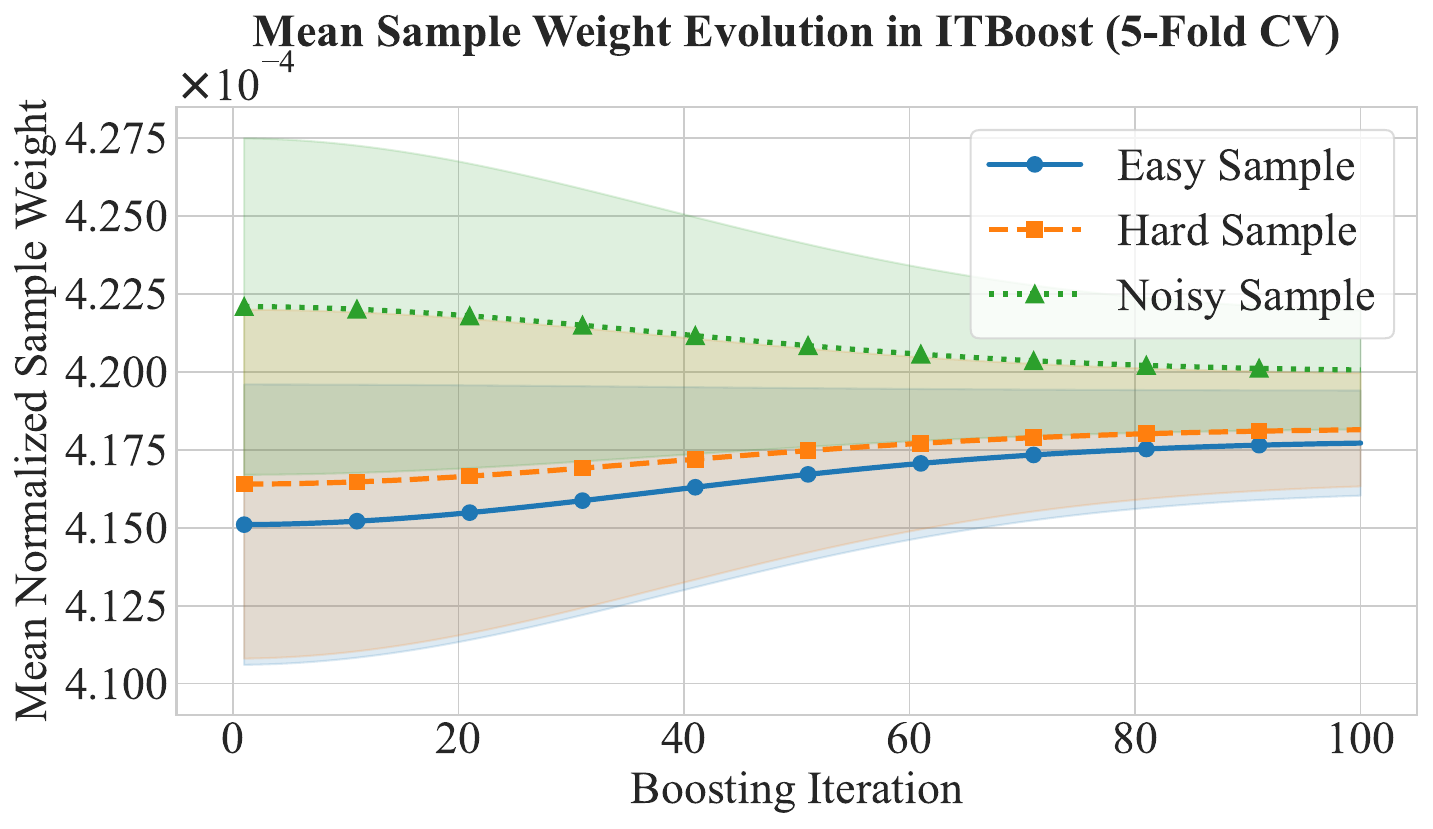}
    \caption{Visualization of information-theoretic trust mechanism for ITBoost: a case study for Jasmine dataset with 20\% label noise.}
    \label{fig:4}
\end{figure}

\subsection{Ablation Study}
\label{sec:ablation}

A core design choice underpinning ITBoost’s robustness lies in the binarization of residual history: this design intentionally discards error magnitude to focus exclusively on the temporal pattern of error directions. We further validate the necessity of this simplification, whether it is critical for robustness or causes detrimental information loss, through an ablation study. Specifically, we compare ITBoost-Binary with ITBoost-Quantized, a variant that retains magnitude information via four-symbol encoding. On clean Jasmine and Bioresponse datasets (\textcolor{blue}{Table \ref{tab:2}}), the two variants delivered comparable performance; ITBoost-Binary even achieved a slight edge while cutting training time by 20–25\%, confirming no compromise to baseline effectiveness. Under label noise conditions (\textcolor{blue}{Figure \ref{fig:3}}), however, ITBoost-Binary outperformed the Quantized variant across all metrics. This performance gap widened as noise intensity increased, which confirms that gradient magnitude from mislabeled samples constitutes a spurious signal. By focusing solely on directional history, the binary formulation better captures the chaotic patterns linked to untrustworthy labels. Notably, this advantage fades under feature noise—highlighting that binarization serves as a targeted, effective defense specifically against label noise. In summary, these results demonstrate that binary residual encoding is not a detrimental simplification, but rather a targeted noise-filtering mechanism that underpins ITBoost’s robust performance.

\begin{table}[t]
    \centering
    \caption{Ablation study for the role of residual binarization in ITBoost under the clean Jasmine and Bioresponse datasets.}
    \label{tab:2}
    \resizebox{\columnwidth}{!}{
    \begin{tabular}{llccccc}
        \toprule
        Dataset & Ablation & ACC & F1 & AUC & Log Loss & Time \\
        \midrule
        \multirow{2}{*}{Jasmine} 
        & Binary (Ours) & \textbf{0.9179} & \textbf{0.9185} & 0.9415 & \textbf{0.1814} & \textbf{13.81} \\
        & Quantized & 0.9149 & 0.9163 & \textbf{0.9507} & 0.1922 & 18.10 \\
        \midrule
        \multirow{2}{*}{Bioresponse} 
        & Binary (Ours) & \textbf{0.9174} & \textbf{0.9238} & \textbf{0.9567} & \textbf{0.1844} & \textbf{102.83} \\
        & Quantized & 0.9153 & 0.9199 & 0.9545 & 0.1916 & 132.87 \\
        \bottomrule
    \end{tabular}
    }
\end{table}

\subsection{Inside ITBoost: A Visual Case Study}
\label{sec:case_study}

To provide a mechanistic understanding of our framework, \textcolor{blue}{Figure \ref{fig:4}} visualizes the average weight evolution for representative sample types on the Jasmine dataset with 20\% noise. The result revealed theoretically consistent trajectories: the easy sample (blue) showed a gently increasing weight as the model solidifies its confidence, while the hard sample (orange) maintained a consistently higher weight, confirming the model correctly focuses on informative boundary points. The most compelling evidence lied with the noisy sample (green). It began with the highest weight due to its large initial error, but as the boosting iterations proceed and its residual history became demonstrably random, its weight entered a steep and continuous decline. This provides direct empirical validation of our core claim: the information-theoretic trust mechanism successfully identifies the noisy sample's chaotic error pattern as untrustworthy and systematically suppresses its influence, preventing it from dominating the learning process. This case study thus provides direct empirical evidence that our trust mechanism successfully differentiates valuable difficulty from misleading noise, validating the core principle behind ITBoost's robustness.

\section{Conclusion}

In this paper, we proposed ITBoost, a framework that redefines sample importance based on the historical trustworthiness of error sources rather than instantaneous residuals. By leveraging the MDL principle to quantify the complexity of residual trajectories, ITBoost effectively distinguishes informative hard examples from chaotic noise, systematically suppressing the influence of corrupted labels. Our theoretical analysis establishes a tighter generalization bound under label noise, while extensive empirical results demonstrate that ITBoost achieves optimal robustness in noisy environments without compromising performance on clean data. We acknowledge that real-time computation of Lempel–Ziv complexity incurs notable overhead, especially on large-scale datasets. Future work will improve scalability via approximate measures, sliding-window evaluation, and periodic trust updates. More broadly, information-theoretic trust opens several promising directions. Leveraging temporal learning dynamics is not limited to boosting: it can be adapted to deep learning for more robust optimizers, applied to curriculum learning to better identify genuinely difficult samples, and extended to online and continual learning to detect and respond to concept drift. By shifting from a static to a dynamic, history-aware notion of sample importance, this work lays the groundwork for a new class of more adaptive and resilient learning algorithms.

\section*{Acknowledgments}
This work was supported in part by the National Natural Science Foundation of China under Grant 62403043. 

\bibliographystyle{named}
\bibliography{ijcai26}

\appendix

\section{Detailed Proofs of Theoretical Results}
\label{sec:appendix_proofs}

In this appendix, we provide the complete mathematical derivations for the theoretical results presented in the main paper.

\subsection{Proof of Lemma 1}
\label{proof:lemma1}

\begin{proof}
We prove the three bounds sequentially.

(i) For the lower bound, consider the function $f(x) = e^{-x}$. Computing the second derivative yields $f''(x) = e^{-x}$. Since $e^{-x} > 0$ for all $x \in \mathbb{R}$, $f(x)$ is a strictly convex function. By Jensen's Inequality, for any random variable $C$, $\mathbb{E}[f(C)] \geq f(\mathbb{E}[C])$. Substituting $f(x) = e^{-x}$ and noting that $\mu = \mathbb{E}[C]$, we directly obtain:
\begin{equation}
    \tau = \mathbb{E}[e^{-C}] \geq e^{-\mathbb{E}[C]} = e^{-\mu}.
\end{equation}

(ii) For the bounded-range upper bound, let us define a centered random variable $X = -(C - \mu)$. Since $C$ is bounded in the interval $[a, b]$, $X$ is bounded in an interval of length $R = b - a$, and by definition $\mathbb{E}[X] = 0$. We can rewrite $\tau$ in terms of $X$:
\begin{equation}
    \tau = \mathbb{E}[e^{-C}] = \mathbb{E}[e^{X - \mu}] = e^{-\mu} \mathbb{E}[e^X].
\end{equation}
Hoeffding's Lemma states that for a random variable $X$ with zero mean bounded in an interval of length $R$, $\mathbb{E}[e^{sX}] \leq \exp(\frac{s^2 R^2}{8})$. Setting $s=1$, we have $\mathbb{E}[e^X] \leq \exp(\frac{R^2}{8})$. Substituting this back into the expression for $\tau$:
\begin{equation}
    \tau \leq e^{-\mu} \cdot \exp\left(\frac{R^2}{8}\right) = \exp\left(-\mu + \frac{R^2}{8}\right).
\end{equation}

(iii) For the sub-Gaussian upper bound, recall that a random variable $Y$ is $\sigma^2$-sub-Gaussian if $\mathbb{E}[e^{\lambda(Y - \mathbb{E}[Y])}] \leq \exp(\frac{\lambda^2 \sigma^2}{2})$ for all $\lambda \in \mathbb{R}$. Letting $Y = C$, we aim to bound $\mathbb{E}[e^{-C}]$. This corresponds to setting $\lambda = -1$ in the moment generating function definition:
\begin{equation}
    \mathbb{E}[e^{-1 \cdot (C - \mu)}] \leq \exp\left(\frac{(-1)^2 \sigma^2}{2}\right) = \exp\left(\frac{\sigma^2}{2}\right).
\end{equation}
Therefore, $\tau = \mathbb{E}[e^{-C}] = e^{-\mu} \mathbb{E}[e^{-(C-\mu)}] \leq e^{-\mu} \exp(\frac{\sigma^2}{2}) = \exp(-\mu + \frac{\sigma^2}{2})$.
\end{proof}

\subsection{Proof of Proposition 1}
\label{proof:prop1}

\begin{proof}
We derive the bound for the ratio $\frac{\tau_{\text{noisy}}}{\tau_{\text{clean}}}$ by simultaneously applying the upper bound to the numerator and the lower bound to the denominator, as established in Lemma 1.

First, for the denominator $\tau_{\text{clean}}$, we apply the Jensen lower bound (Lemma 1.i):
\begin{equation}
    \tau_{\text{clean}} \geq \exp(-\mu_{\text{clean}}) \implies \frac{1}{\tau_{\text{clean}}} \leq \exp(\mu_{\text{clean}}).
\end{equation}

Second, for the numerator $\tau_{\text{noisy}}$, we apply the generalized upper bound (Lemma 1.ii or .iii). Let $\text{corr}$ represent the correction term ($\frac{R^2}{8}$ for the bounded case or $\frac{\sigma^2}{2}$ for the sub-Gaussian case). Then:
\begin{equation}
    \tau_{\text{noisy}} \leq \exp(-\mu_{\text{noisy}} + \text{corr}).
\end{equation}

Combining these two inequalities yields:
\begin{align}
    \frac{\tau_{\text{noisy}}}{\tau_{\text{clean}}} &\leq \exp(-\mu_{\text{noisy}} + \text{corr}) \cdot \exp(\mu_{\text{clean}}) \\
    &= \exp(-(\mu_{\text{noisy}} - \mu_{\text{clean}}) + \text{corr}).
\end{align}
Substituting the complexity separation gap $\Delta_c = \mu_{\text{noisy}} - \mu_{\text{clean}}$, we obtain the final bound:
\begin{equation}
    \frac{\tau_{\text{noisy}}}{\tau_{\text{clean}}} \leq \exp(-\Delta_c + \text{corr}).
\end{equation}
\end{proof}

\subsection{Proof of Theorem 1}
\label{proof:thm1}

\begin{proof}
We analyze the expected change in empirical risk for one boosting iteration $t$. Let the ensemble be updated as $F_{t+1}(x) = F_t(x) + \nu h_t(x)$. For theoretical tractability, we assume a learning rate $\nu=1$. The empirical risk is $R_t(F) = \frac{1}{n} \sum_{i=1}^n \mathcal{L}(y_i, F_t(x_i))$.
Using the convexity and $L$-Lipschitz continuity of the loss $\mathcal{L}$, the reduction in risk is bounded by the correlation between the negative gradient (pseudo-residual $g_i$) and the weak learner $h_t(x_i)$. Taking the expectation over the data distribution $\mathcal{D}$:
\begin{equation}
    \mathbb{E}[R_{t+1} - R_t] \leq -\mathbb{E}_{(x,y) \sim \mathcal{D}} \left[ \langle \nabla \mathcal{L}, h_t(x) \rangle \right].
\end{equation}
In ITBoost, the weak learner $h_t$ is trained to fit the pseudo-residuals $g_i$ weighted by $w_i = |g_i| \cdot \tau_i$, where $\tau_i = g(C(\mathbf{s}_i)) = e^{-C(\mathbf{s}_i)}$. Ideally, $h_t(x)$ aligns with the direction of the trust-weighted gradient. Thus, the expected descent is proportional to the trust weights:
\begin{equation}
    \mathbb{E}[R_{t+1} - R_t] \propto -\mathbb{E}[\tau_i].
\end{equation}
Let $\alpha$ be the noise rate. We decompose the expected trust $\mathbb{E}[\tau]$ into clean and noisy components:
\begin{equation}
    \mathbb{E}[\tau] = (1-\alpha)\mathbb{E}[\tau | \text{clean}] + \alpha \mathbb{E}[\tau | \text{noisy}].
\end{equation}
To quantify the improvement over a non-robust baseline (which implicitly assumes $\mathbb{E}[\tau | \text{clean}] \approx \mathbb{E}[\tau | \text{noisy}]$), we examine the difference in trust generated by the complexity gap. Since the trust mapping function $g(c) = e^{-c}$ is $L_g$-Lipschitz, we have:
\begin{equation}
    |g(C_{\text{clean}}) - g(C_{\text{noisy}})| \leq L_g |C_{\text{clean}} - C_{\text{noisy}}|.
\end{equation}
Taking expectations on both sides and applying Assumption 2 ($\mathbb{E}[C_{\text{noisy}}] - \mathbb{E}[C_{\text{clean}}] = \Delta_c$), and noting that $g(c)$ is a monotonically decreasing function (implying $\tau_{\text{clean}} > \tau_{\text{noisy}}$ for $\Delta_c > 0$):
\begin{equation}
    \mathbb{E}[\tau_{\text{clean}}] - \mathbb{E}[\tau_{\text{noisy}}] \geq L_g (\mathbb{E}[C_{\text{noisy}}] - \mathbb{E}[C_{\text{clean}}]) = L_g \Delta_c.
\end{equation}
This inequality implies that the algorithm assigns significantly higher weight to clean samples (proportional to $L_g \Delta_c$). The contribution to the risk reduction is scaled by the Lipschitz constant of the loss ($L$) and the bound of the weak learner ($B$). Thus, the expected risk reduction due to complexity separation is bounded by:
\begin{equation}
    \text{Descent Improvement} \leq - L \cdot B \cdot (L_g \Delta_c).
\end{equation}
Finally, to relate the expected risk to the empirical risk observed on a finite sample of size $n$, we invoke standard generalization bounds based on Rademacher complexity $\mathfrak{R}_n(\mathcal{F})$, which scales as $O(1/\sqrt{n})$. Adding this generalization gap yields the final bound:
\begin{equation}
    \mathbb{E}[R_{t+1}(F) - R_t(F)] \leq -L L_g B \Delta_c + O\left(\frac{1}{\sqrt{n}}\right).
\end{equation}
This confirms that a larger complexity gap $\Delta_c$ directly accelerates the reduction of generalization risk.
\end{proof}

\section{Complexity Analysis}
\label{sec:complexity}

In this section, we analyze the time and space complexity of the ITBoost algorithm and compare it to a standard GBDT implementation.
Let $N$ be the number of training samples, $M$ be the number of boosting iterations (estimators), $d$ be the number of features, and $D$ be the maximum depth of the decision tree weak learners.

\subsection{Time Complexity}

\textbf{1) Standard GBDT.} The main computational cost of a GBDT lies within its training loop. In each of the $M$ iterations, the primary operations are:
\begin{itemize}
    \item Computing gradients for all samples: $\mathcal{O}(N)$.
    \item Training a decision tree of depth $D$: The complexity of building a tree is typically $\mathcal{O}(N \cdot d \cdot D)$ for finding the best splits.
    \item Updating the model predictions: $\mathcal{O}(N)$.
\end{itemize}
The tree-building step is dominant. Therefore, the total time complexity for a standard GBDT is $M \cdot \mathcal{O}(N \cdot d \cdot D)$.

\textbf{2) ITBoost.} Our method introduces several additional steps within each iteration $m \in [1, M]$:
\begin{itemize}
    \item Updating residual histories: Appending a character to $N$ strings, which is an $\mathcal{O}(N)$ operation.
    \item Computing LZ complexity: This is the most significant new step. A standard implementation of the Lempel-Ziv algorithm on a string of length $L$ has a complexity of $\mathcal{O}(L)$. At iteration $m$, the history string for each sample has length $m$. We perform this for all $N$ samples, resulting in a complexity of $\mathcal{O}(N \cdot m)$.
    \item Normalizing and computing weights: This involves element-wise operations, taking $\mathcal{O}(N)$.
\end{itemize}
Thus, the additional cost for ITBoost at iteration $m$ is $\mathcal{O}(N \cdot m)$. The total time complexity is the sum of the GBDT cost and this additional cost summed over all iterations: $\sum_{m=1}^M (\mathcal{O}(N \cdot d \cdot D) + \mathcal{O}(N \cdot m)) = M \cdot \mathcal{O}(N \cdot d \cdot D) + \mathcal{O}(N \cdot M^2)$.

\textbf{Comparison and Empirical Validation.} The additional time complexity introduced by ITBoost is $\mathcal{O}(N \cdot M^2)$. This term becomes significant when the number of boosting iterations $M$ is large. For typical scenarios where $M$ is in the low hundreds, this overhead is manageable, but for very deep ensembles, it can become the dominant factor. We empirically validate this in \textcolor{blue}{Table \ref{tab:time}}, which compares the wall-clock training time of ITBoost against baselines. The results confirmed that our analysis. ITBoost incurred a notable training time overhead, particularly on datasets like Madelon and Bioresponse where the number of samples $N$ amplifies the cost of the trust-assessment loop. For instance, on the Jasmine dataset, ITBoost was considerably slower than the highly optimized LGBM (14.26s vs. 0.22s). This highlights the primary trade-off of our method: enhanced robustness is achieved at the cost of computational speed, positioning ITBoost as a strategic choice for applications where model resilience is paramount.

\begin{table}[h]
    \centering
    \caption{Comparison of wall-clock time (s).}
    \label{tab:time}
    \resizebox{\columnwidth}{!}{
    \begin{tabular}{ccccccccc}
        \toprule
        Dataset & ADA & GBDT & XGB & LGBM & CAT & NGB & TabPFN & ITBoost \\
        \midrule
        OVA\_Uterus & 21.97 & 54.10 & 3.86 & 3.77 & 110.27 & 50.89 & 19.33 & 70.48 \\
        Madelon & 5.43 & 13.93 & 1.68 & 1.14 & 10.62 & 16.07 & 4.56 & 74.66 \\
        Jasmine & 1.10 & 1.20 & 0.15 & 0.22 & 2.59 & 2.73 & 2.22 & 14.26 \\
        Bioresponse & 9.46 & 19.63 & 2.46 & 1.36 & 35.76 & 23.32 & 11.81 & 108.64 \\
        Creditcard & 0.97 & 1.03 & 0.10 & 0.19 & 0.92 & 1.41 & 1.54 & 5.56 \\
        \bottomrule
    \end{tabular}
    }
\end{table}

\subsection{Space Complexity}

In terms of space complexity, a standard GBDT primarily requires memory to store the dataset, $\mathcal{O}(N \cdot d)$, and the final ensemble of trees $\mathcal{O}(M \cdot 2^D)$ (Storing the ensemble of M trees. The size of a tree of depth $D$ is roughly $\mathcal{O}(2^D)$). However, ITBoost has one significant additional memory requirement, storing residual histories. We must maintain a history string for each of the $N$ samples. At the end of training (iteration $M$), each string has a length of $M$. The additional space complexity is therefore $\mathcal{O}(N \cdot M)$.

\textbf{Comparison and Empirical Validation.} The requirement to store the full residual history for every sample introduces a memory overhead that scales with both the number of samples and the number of boosting iterations. For large datasets and deep ensembles, this can lead to a substantial increase in memory consumption compared to standard GBDT. We measured this empirically in \textcolor{blue}{Table \ref{tab:memory}}. The results showed that ITBoost's additional space requirement is modest in practice. Its peak memory usage was consistently comparable to standard Scikit-learn implementations like GBDT and ADA, and significantly lower than memory-intensive models like NGB on some datasets. For example, on the Bioresponse dataset, ITBoost's memory footprint (21.21 MB) was nearly identical to GBDT's (20.69 MB). This suggests that for typical ensemble sizes, the $\mathcal{O}(N \cdot M)$ overhead is manageable and does not pose a critical barrier to deployment on standard hardware.

\begin{table}[h]
    \centering
    \caption{Comparison of peak memory usage (MB).}
    \label{tab:memory}
    \resizebox{\columnwidth}{!}{
    \begin{tabular}{ccccccccc}
        \toprule
        Dataset & ADA & GBDT & XGB & LGBM & CAT & NGB & TabPFN & ITBoost \\
        \midrule
        OVA\_Uterus & 8.69 & 8.52 & 0.10 & 4.73 & 1.19 & 54.12 & 96.16 & 8.55 \\
        Madelon & 4.28 & 4.23 & 0.24 & 1.93 & 0.05 & 24.37 & 125.64 & 4.56 \\
        Jasmine & 1.64 & 1.60 & 0.25 & 1.88 & 0.16 & 8.32 & 43.50 & 1.96 \\
        Bioresponse & 20.72 & 20.69 & 0.28 & 2.14 & 0.71 & 123.11 & 270.47 & 21.21 \\
        Creditcard & 0.37 & 0.21 & 0.12 & 1.83 & 0.19 & 0.81 & 3.82 & 0.36 \\
        \bottomrule
    \end{tabular}
    }
\end{table}

In summary, the analysis reveals that the enhanced robustness of ITBoost comes at the cost of increased time and space complexity, both scaling with the number of boosting iterations $M$. The primary trade-off is one of computational resources versus model resilience.

\section{Statistic Analysis}
\label{sec:statistic}

To rigorously assess the statistical significance of our findings, we conducted a Friedman test on the performance ranks of all algorithms across the five datasets, with the results detailed in \textcolor{blue}{Table \ref{tab:friedman}}. The test reveals a significant difference among the models (p-value $< 0.05$) across all metrics. Additionally, ITBoost achieved a best final rank of 1.0 in every category, establishing it as the superior method in this comparison. This statistical dominance is not incidental, it provides strong quantitative evidence for our central thesis that by assessing the historical trustworthiness of gradients, ITBoost avoids overfitting to spurious patterns that limit the performance of leading GBDT methods, thereby validating its effectiveness.

\begin{table}[t]
    \centering
    \caption{Friedman test for ITBoost compared with the SOTA ensembles across five benchmark datasets}
    \label{tab:friedman}
    \setlength{\tabcolsep}{1.5pt}
    \resizebox{\columnwidth}{!}{
    \begin{tabular}{clllllllll}
        \toprule
        \multicolumn{2}{c}{Friedman Test} 
        & ADA & GBDT & XGB & LGBM & CAT & NGB & TabPFN & ITBoost \\
        \midrule

        \multirow{4}{*}{\rotatebox{90}{ACC}} 
        & Average    
        & 0.8073 & 0.8373 & 0.8553 & 0.8621 & 0.8669 & 0.8183 & 0.8857 & \textbf{0.9344} \\
        & Mean Rank 
        & 6.6 & 5.8 & 5.4 & 3.7 & 3.4 & 7.1 & 3.0 & \textbf{1.0} \\
        & Final Rank
        & 7 & 6 & 5 & 4 & 3 & 8 & 2 & \textbf{1} \\
        & p-value   
        & \multicolumn{8}{l}{0.000700 $< 0.05$} \\
        \midrule

        \multirow{4}{*}{\rotatebox{90}{F1}} 
        & Average    
        & 0.8151 & 0.8450 & 0.8621 & 0.8692 & 0.8748 & 0.8299 & 0.8918 & \textbf{0.9360} \\
        & Mean Rank 
        & 6.4 & 6.4 & 5.6 & 3.8 & 3.3 & 6.5 & 3.0 & \textbf{1.0} \\
        & Final Rank
        & 6 & 7 & 5 & 4 & 3 & 8 & 2 & \textbf{1} \\
        & p-value   
        & \multicolumn{8}{l}{0.001091 $< 0.05$} \\
        \midrule

        \multirow{4}{*}{\rotatebox{90}{AUC}} 
        & Average    
        & 0.8561 & 0.8909 & 0.9133 & 0.9191 & 0.9209 & 0.8773 & 0.9410 & \textbf{0.9623} \\
        & Mean Rank 
        & 7.2 & 6.2 & 5.0 & 3.0 & 3.0 & 7.4 & \textbf{1.8} & 2.4 \\
        & Final Rank
        & 7 & 6 & 5 & 3 & 4 & 8 & \textbf{1} & 2 \\
        & p-value   
        & \multicolumn{8}{l}{0.000133 $< 0.05$} \\
        \midrule

        \multirow{4}{*}{\rotatebox{90}{Log Loss}} 
        & Average    
        & 0.5604 & 0.4451 & 0.3775 & 0.3381 & 0.3211 & 0.4513 & 0.2799 & \textbf{0.1547} \\
        & Mean Rank 
        & 7.6 & 5.4 & 5.8 & 4.6 & 2.4 & 6.6 & 2.2 & \textbf{1.4} \\
        & Final Rank
        & 8 & 5 & 6 & 4 & 3 & 7 & 2 & \textbf{1} \\
        & p-value   
        & \multicolumn{8}{l}{0.000100 $< 0.05$} \\
        \bottomrule
    \end{tabular}
    }
\end{table}

\section{Related Work}

The modern era of boosting was ushered in by Friedman \textcolor{blue}{\shortcite{friedman2001greedy}}, who re-framed the AdaBoost algorithm as a form of functional gradient descent. This insight gave rise to the GBDT, a flexible framework for constructing additive models by iteratively fitting weak learners to the negative gradient of a loss function. This foundational concept has been the subject of extensive research, leading to highly optimized and powerful implementations that represent the SOTA for tabular data. Key advancements include the introduction of regularization techniques, second-order optimization methods to better guide the learning steps, and efficient handling of sparse data, as famously combined in XGB \textcolor{blue}{\cite{chen2015xgboost}}. Subsequent work has focused primarily on computational efficiency and categorical feature handling. LGBM \textcolor{blue}{\cite{ke2017lightgbm}} introduced histogram-based split finding and gradient-based one-side sampling for massive speedups, while CAT \textcolor{blue}{\cite{prokhorenkova2018catboost}} developed ordered boosting and sophisticated methods for processing categorical features. Despite these significant improvements in speed, regularization, and usability, the core optimization process remains unchanged: they are all driven by the magnitude of the gradient. This shared reliance makes them all fundamentally susceptible to the influence of label noise.

A major line of research for improving the robustness of boosting focuses on designing robust loss functions that mitigate the influence of label noise and outliers. Early studies introduced symmetric or truncated non-convex losses—such as the Savage loss—to cap the impact of large errors, theoretically proving their resistance to symmetric noise and developing algorithms like SavageBoost \textcolor{blue}{\cite{masnadi2008design}}. To further prevent boosting algorithms from overemphasizing mislabeled ``hard'' samples, Freund \textcolor{blue}{\shortcite{freund2009robust}} proposed BrownBoost and RobustBoost, which employ dynamically changing non-convex potential functions that effectively ``give up'' on persistently difficult examples, thereby improving robustness under noisy conditions. In recent years, several practical extensions have been explored: truncated or ramp-based robust boosting methods (e.g., truncated Huber loss) limit gradients for large residuals and achieve stable performance on datasets with outliers \textcolor{blue}{\cite{wang2018robust}}; SPLBoost integrates self-paced learning with boosting to gradually include samples during training, distinguishing between easy, hard, and noisy examples to enhance label-noise tolerance \textcolor{blue}{\cite{wang2019splboost}}; and more recently, Robust-GBDT extends non-convex robust losses to gradient-boosted decision trees, improving robustness in tabular data tasks \textcolor{blue}{\cite{luo2025robust}}.

However, despite their theoretical contributions, widespread adoption of methods like BrownBoost, RobustBoost, and SPLBoost has been hindered by the lack of officially maintained, production-ready implementations compatible with modern Python ecosystems. Consequently, standard practice in recent robust boosting literature (e.g., \textcolor{blue}{\cite{luo2025robust}}) prioritizes benchmarking against widely adopted frameworks such as XGB and LGBM rather than these specialized variants. Furthermore, while effective at capping the influence of outliers, these existing robust methods remain fundamentally “magnitude-based.” They operate under the assumption that large errors equate to noise, failing to differentiate between a genuinely hard, informative example and a noisy one.
\end{document}